\documentclass{elsarticle}
\usepackage{amsmath,amssymb,amsfonts}
\usepackage{times}
\usepackage{indentfirst}
\usepackage{graphicx}
\usepackage{subfigure}
\usepackage{amssymb}
\newtheorem{theorem}{Theorem}
\newtheorem{definition}{Definition}
\newtheorem{proposition}{Proposition}
\newtheorem{corollary}{Corollary}
\newtheorem{lemma}{Lemma}
\newtheorem{example}{Example}
\newdefinition{remark}{Remark}
\newproof{proof}{Proof}

\begin{document}
\begin{frontmatter}
\title{Dependence space of matroids and its application to attribute reduction}
\author[addr1]{Aiping Huang}
\author[addr2]{William Zhu\corref{cor1}}\ead{williamfengzhu@gmail.com}

\cortext[cor1]{Corresponding author.}
\address[addr1]{Tan Kah Kee College, \\Xiamen University, Zhangzhou 363105, China}
\address[addr2]{Lab of Granular Computing,\\Minnan Normal University, Zhangzhou 363000, China}

\begin{abstract}
Attribute reduction is a basic issue in knowledge representation and data mining.
Rough sets provide a theoretical foundation for the issue.
Matroids generalized from matrices have been widely used in many fields, particularly greedy algorithm design, which plays an important role in attribute reduction.
Therefore, it is meaningful to combine matroids with rough sets to solve the optimization problems.
In this paper, we introduce an existing algebraic structure called dependence space to study the reduction problem in terms of matroids.
First, a dependence space of matroids is constructed.
Second, the characterizations for the space such as consistent sets and reducts are studied through matroids.
Finally, we investigate matroids by the means of the space and present two expressions for their bases.
In a word, this paper provides new approaches to study attribute reduction.
\end{abstract}


\end{frontmatter}

\section{Introduction}

In many applications, information and knowledge are stored and represented in an information table, where an object is described by a set of attributes.
It is nature for us to fact such a problem that for a special property whether all the attributes in the attribute set are always necessary to preserve this property or not, because using entire attribute set to describe the property is time consuming, and obtained rules may be difficult to understand, to apply and to verify.
In order to solve these problems, attribute reduction is required~\cite{DaiWangTianLiu13Attribute,FanZhu12Attribute,MinHeQianZhu11Test,WeiLiZhang07Knowledge}.
Rough sets proposed by Pawlak~\cite{Pawlak82Rough} may be the most recent one marking significant contributions to deal with the issues of knowledge reduction in the sense of reducing attributes.
It can describe knowledge via set-theoretic analysis based on equivalence classification for the universe set.
On one hand, in order to solve practical problems preferably through rough sets, its axiom systems have been built~\cite{HooshmandaslKarimiAlmbardarDavvaz13Axiomatic,Liu08Axiomatic,Liu13Using} and generalization works have been done~\cite{ZhuWang03Reduction,Zhu07Generalized,Zhu09RelationshipBetween}.

On the other hand, as is known, many optimization issues related to rough sets, including attribute reduction, are NP-hard, and thus, typically require greedy algorithms.
To find effective methods for solving these problems, rough sets have been combined with other theories~\cite{DaiTian13Fuzzy,DzikJarvinenKondo13Representing,WangZhu12Quantitative,WangZhu13Equivalent}, especially matroid theory~\cite{HuangZhu12Geometric,LiLiu12Matroidal,TangSheZhu12Matroidal,WangZhuZhuMin12Matroidal,WangZhu13Four,ZhuWang13Rough}, which borrows extensively from linear algebra and graph theory.
With abundant theories and a perfect system, matroid theory has been widely used in many fields including combinatorial optimization, network flows~\cite{Lawler01Combinatorialoptimization}, and algorithm design, especially greedy algorithm design~\cite{Edmonds71Matroids}.
Therefore, studying rough sets in conjunction with matroids may help solve some optimization issues.
A greedy algorithm is an algorithm that follows the problem solving heuristic of making the locally optimal choice at each stage with the hope of finding a global optimum.

In many problems, a greedy strategy does not in general produce an optimal solution, but nonetheless a greedy heuristic may yield locally optimal solutions that approximate a global optimal solution in a reasonable time.
The most wonderful thing is that the question whether a greedy algorithm produces an optimal solution for a particular problem can be converted to the question whether there exists a translation of the problem into a matroid.
Furthermore, the optimal solutions are bases of the matroid.
There are two well-known optimization problems solved by greedy algorithms designed by matroids.
One is the minimum-weight spanning tree problem, the other is job assignment problem.
The solution for the former problem is a base of the corresponding graphic matroid, and the solution for the latter one is a base of the corresponding transversal matroid.
This prompts us to establish relationships between the bases of a matroid and the attribute reduction.

In this paper, we construct a dependence space in the context of matroids and apply it to attribute reduction problems.
First, a dependence space of matroids is proposed from the viewpoint of closure operator.
Second, we study the dependence space by means of matroids.
It is interesting to find that the set of consistent sets and the set of reducts of the dependence space are the family of independent sets and the family of bases of the corresponding matroid, respectively.
Finally, matroids are studied conversely by dependence spaces and two expressions for bases of matroids are presented.
Therefore, this work provides new viewpoints for studying the issues of attribute reduction.

The rest of this paper is arranged as follows.
Section \ref{S:Preliminaries} reviews some fundamental concepts related to rough sets, information systems, and matroids.
In Section \ref{S:Dependencespaceinducedbymatroids}, we propose a dependence space of matroids.
Section \ref{S:Matroidalapproachtodependencespaces} studies the dependence space in terms of matroids.
The dependence spaces are studied matroids in Section \ref{S:Dependencespaceapproachtomatroids},
while Section \ref{S:conclusions} concludes this paper.

\section{Preliminaries}
\label{S:Preliminaries}

To facilitate our discussion, some fundamental concepts related to rough sets and matroids are reviewed in this section.

\subsection{Rough sets}
Rough sets, based on equivalence relations, provide a systematic approach to data preprocessing
in data mining.
The lower and upper approximation operations, which are two key concepts in the theory, are used to describe objects.

\begin{definition}(Approximation operators~\cite{Pawlak82Rough})
Let $U$ be a finite set, $R$ be an equivalence relation of $U$, and $X \subseteq U$.
Then the lower and upper approximations of $X$, denoted by $R_{\ast}(X)$ and $R^{\ast}(X)$, respectively, are defined as:
\begin{center}
$~~~~R_{\ast}(X) = \{x \in U: \forall~y \in U, xRy \Rightarrow y \in X\} = \{x \in U: [x]_{R} \subseteq X\}$,\\
$R^{\ast}(X) = \{x \in U: \exists~y \in U~s.t.~xRy\} = \{x \in U: [x]_{R} \bigcap X \neq \emptyset\}$,
\end{center}
where $[x]_{R} = \{y \in U : xRy\}$ denotes the equivalence class of $x$ with respect to relation $R$.
\end{definition}

\subsection{Matroids}

Matroid theory borrows extensively from the terminology of linear algebra and graph theory,
largely because it is an abstraction of various notions of central importance in these fields, such as independent sets, bases, and rank functions.
One of the most valuable definitions of matroids is presented in terms of independent sets.

\begin{definition}(Matroid~\cite{Oxley93Matroid})
\label{D:thedefinitionofmatroidfromindependentsets}
A matroid $M$ is an ordered pair $(U,\mathcal{I})$ consisting of a finite set $U$ and a collection $\mathcal{I}$ of subsets of $U$ satisfying
the following three conditions:\\
(I1) $\emptyset \in \mathcal{I}$.\\
(I2) If $I \in \mathcal{I}$ and $I^{'} \subseteq I$, then $I^{'} \in \mathcal{I}$.\\
(I3) If $I_{1},I_{2} \in \mathcal{I}$ and $|I_{1}| < |I_{2}|$,
then there is an element $e \in I_{2} - I_{1}$ such that $I_{1} \bigcup \{e\} \in \mathcal{I}$, where $|X|$ denotes the cardinality of $X$.
\end{definition}

The members of $\mathcal{I}$ are the independent sets of $M$ and $U$ is the ground set of $M$.
We often write $\mathcal{I}(M)$ for $\mathcal{I}$, particularly when several matroids are being considered.
The bases and rank function of a matroid are defined based on independent sets.
To describe these concepts intuitively, certain denotations are presented.

\begin{definition}\cite{Oxley93Matroid}
Let $\mathcal{A}$ be a family of subsets of $U$.
Then we denote\\
$Max(\mathcal{A}) = \{X \in \mathcal{A}: \forall~Y \in \mathcal{A}, if~X \subseteq Y, then ~X = Y\}$,\\
$Min(\mathcal{A}) = \{X \in \mathcal{A}: \forall~Y \in \mathcal{A}, if~Y \subseteq X, then ~X = Y\}$.
\end{definition}

First, the bases of a matroid, generalized from the maximal linearly independent group in vector space, are defined as follows.

\begin{definition}(Base~\cite{Oxley93Matroid})
Let $M = (U, \mathcal{I})$ be a matroid.
Then a subset of $U$ is said to be a base of $M$ if it is maximal in the sense that it is not contained in any other element of $\mathcal{I}$.
If we denote the collection of bases of $M$ by $\mathcal{B}(M)$, $\mathcal{B}(M) = Max(\mathcal{I})$.
\end{definition}

The rank function $r_{M}: 2^{U}\rightarrow \mathbf{N}$ of a matroid, which is a generalization of the rank of a matrix, is defined as $r_{M}(X) = max\{|I|: I \subseteq X, I \in \mathcal{I}\}$ ($X \subseteq U$).
The value $r_{M}(X)$ is called the rank of $X$ in $M$.
Based on the rank function, the closure operator $cl_{M}: 2^{U} \rightarrow 2^{U}$ of a matroid is defined as $cl_{M}(X) = \{x \in U: r_{M}(X) = r_{M}(X \bigcup \{x\})\}$ ($X \subseteq U$).
If $cl_{M}(X) = X$, then $X$ is called a closed set of matroid $M$.
If $X$ is a closed set and $r_{M}(X) = r_{M}(U) - 1$, then $X$ is called a hyperplane of matroid $M$ and the set of all hyperplanes of $M$ is denoted by $\mathcal{H}(M)$.
It is clear that the ground set of $M$ is a closed set of $M$.
The following proposition defines a matroid in terms of closure operators.

\begin{proposition}(Closure axiom~\cite{Oxley93Matroid})
\label{P:closureaxiom}
Let $U$ be a set.
Then the function $cl: 2^{U} \rightarrow 2^{U}$ is the closure operator of a matroid if and only if $cl$ satisfies the following conditions:\\
$(CL1):$ If $X \subseteq U$, then $X \subseteq cl(X)$.\\
$(CL2):$ If $X \subseteq Y \subseteq U$, then $cl(X) \subseteq cl(Y)$.\\
$(CL3):$ If $X \subseteq U$, then $cl(cl(X)) = cl(X)$.\\
$(CL4):$ If $X \subseteq U$, $x \in U$ and $y \in cl(X \bigcup \{x\}) - cl(X)$, then $x \in cl(X \bigcup \{y\})$.
\end{proposition}

Similar to the operation from a vector space to its subspace, the restriction matroid is defined.

\begin{definition}(Restriction matroid~\cite{Oxley93Matroid})
Let $M = (U, \mathcal{I})$ be a matroid.
Then for any $X \subseteq U$, the order pair $(X, \mathcal{I}|X)$ is a matroid, where $\mathcal{I}|X = \{I \subseteq X: I \in \mathcal{I}\}$.
This is called a restriction of $M$ to $X$ and is denoted by $M|X$.
\end{definition}

\section{Dependence space induced by matroids}
\label{S:Dependencespaceinducedbymatroids}
The information about the objects of an information system yielded by different sets of attributes may depend on each other in various ways.
For example, it may turn out that a proper subset of a set of attributes classifies the objects with the same accuracy as the original sets.
Dependence spaces were introduced by Novotn$\acute{y}$ and Pawlak~\cite{NovotnyPawlak91Algebraic} as a general abstract setting for studying such informational dependency.
Matroids provide well-established platforms for greedy algorithms, and just for this reason, matroids arise naturally in a number of problems in combinatorial optimization.
In order to give play to the two theories' respective advantages in attribute reduction, this section constructs a dependence space by matroids firstly.
At the beginning of this section, the concept of dependence space is presented.

\begin{definition}(Dependence space~\cite{NovotnyPawlak91Algebraic})
Let $U$ be a nonempty set and $\Theta$ an equivalence relation of $2^{U}$.
For all $B_{1}, B_{2}, C_{1}, C_{2} \in 2^{U}$,
\begin{center}
    $(B_{1}, C_{1}) \in \Theta, (B_{2}, C_{2}) \in \Theta \Rightarrow (B_{1} \bigcup B_{2}, C_{1} \bigcup C_{2}) \in \Theta$.
\end{center}
Then $\Theta$ is referred to as a congruence relation of $2^{U}$ and the pair $(U, \Theta)$ is called a dependence space.
\end{definition}

The dependence space in compliance with the above definition is a pair containing a nonempty set and a congruence relation.
In other words, through constructing a congruence relation, one can obtain a dependence space.

\begin{definition}
\label{D:arelationinducedbymatroids}
Let $M$ be a matroid of $U$.
Then one can define a relation of $2^{U}$ as follows:
For all $X, Y \in 2^{U}$,
\begin{center}
    $(X, Y) \in \Theta_{M} \Leftrightarrow cl_{M}(X) = cl_{M}(Y)$.
\end{center}
\end{definition}

It is clear that the relation is an equivalence relation.

\begin{example}
\label{E:example1}
Let $M = (U, \mathcal{I})$ be a matroid, where $U = \{1, 2, 3\}$ and $\mathcal{I} = \{ \emptyset, \{1\}, \{2\}, $ $\{3\},$ $ \{1, 2\},$ $ \{2, 3\}\}$.
Then utilizing the definition of closure, we have $cl_{M}(\emptyset) = \emptyset$, $cl_{M}(\{2\}) = \{2\}$, $cl_{M}(\{1\}) = cl_{M}(\{3\}) = cl_{M}(\{1, 3\}) = \{1, 3\}$ and $cl_{M}(\{1, 2\}) = cl_{M}(\{2, 3\}) = cl_{M}(U) = U$.
Thus, $\Theta_{M} = \{(\emptyset, \emptyset),$ $(\{1\}, \{1\}),$ $(\{1\}, \{3\}),$ $(\{3\}, \{1\}),$ $(\{1\}, \{1, 3\}),$ $ (\{\{3\}, \{3\}),$ $(\{3\}, \{1, 3\}),$ $ (\{1, 3\}, \{1\}),$ $(\{1, 3\}, \{1, 3\}),$ $(\{1, 3\}, \{3\}),$ $(\{2\}, \{2\}),$ $(\{1, 2\}, \{1, 2\}),$ $(\{1, 2\}, \{2, 3\}),$ $(\{1, 2\}, U),$ $(\{2, 3\}, \{1, 2\}),$ $(\{2, 3\}, \{2, 3\}),$ $(\{2, 3\}, U),$ $(U, \{1, 2\}),$ $(U, \{2, 3\}),$ $(U, U)\}$.
\end{example}

In order to verify whether the relation is a congruence relation or not, we study certain properties of closed sets in matroids first.

\begin{lemma}\cite{Oxley93Matroid}
\label{L:elementoflatticeexpressedbyhyperplanes}
Let $X$ be a closed set of a matroid $M$ which ground set is $U$ and suppose that $r_{M}(X) = r_{M}(U) - k$ where $k \geq 1$.
Then $M$ has a set $\{H_{1}, H_{2}, \cdots, H_{k}\}$ of hyperplanes such that $X = \bigcap_{i = 1}^{k} H_{i}$.
\end{lemma}

Based on the result, we find that any closed set whose rank is less than that of ground set can be expressed by the intersection of some hyperplanes containing the set.

\begin{corollary}
\label{C:aclosuresetexpressedbyhyperplane}
Let $X$ be a closed set of a matroid $M$ which ground set is $U$ and suppose that $r_{M}(X) < r_{M}(U)$.
Then $X = \bigcap \{H \in \mathcal{H}(M): X \subseteq H\}$.
\end{corollary}

\begin{proof}
It is obvious that $X \subseteq \bigcap\{H \in \mathcal{H}(M): X \subseteq H\}$.
Let us assume that $X = \bigcap_{i = 1}^{k} H_{i}$, utilizing Lemma \ref{L:elementoflatticeexpressedbyhyperplanes}, we have $\{H_{i} \in \mathcal{H}(M): i = 1, 2, \cdots, k\} \subseteq \{H \in \mathcal{H}(M): X \subseteq H\}$, thus $X \subseteq \bigcap\{H \in \mathcal{H}(M): X \subseteq H\} \subseteq \bigcap_{i = 1}^{k} H_{i} = X$.
Therefore, $X = \bigcap \{H \in \mathcal{H}(M): X \subseteq H\}$.
\end{proof}

In fact, the result can be generalized to the closure of any subset.
In other words, the closures of subsets can be characterized by hyperplanes.

\begin{proposition}
\label{P:theotherclosureexpressionbyhyperplane}
Let $M$ be a matroid of $U$.
Then for all $X \subseteq U$,
\begin{align}
cl_{M}(X) = \left\{\begin{aligned}%
&U && \mbox r_{M}(X) = r_{M}(U),\\
&\bigcap \{H \in \mathcal{H}(M): X \subseteq H\} && \mbox r_{M}(X) \neq r_{M}(U).\\
\end{aligned}\right.
\end{align}
\end{proposition}

\begin{proof}
First, we prove $cl_{M}(X) = U$ when $r_{M}(X) = r_{M}(U)$.
For all $x \in U$, $r_{M}(X) \leq r_{M}(X \bigcup \{x\}) \leq r_{M}(U) = r_{M}(X)$ which implies $r_{M}(X) = r_{M}(X \bigcup \{x\})$, i.e., $x \in cl_{M}(X)$.
Hence $U \subseteq cl_{M}(X)$.
Combining with $cl_{M}(X) \subseteq U$, we have proved the result.
For all $X \subseteq U$, $r_{M}(cl_{M}(X)) = r_{M}(X) \leq r_{M}(U)$.
Then $r_{M}(X) \neq r_{M}(U)$ implies $r_{M}(cl_{M}(X)) < r_{M}(U)$.
Utilizing Corollary \ref{C:aclosuresetexpressedbyhyperplane}, we know $\{H\in \mathcal{H}(M): cl_{M}(X) \subseteq H\} \neq \emptyset$.
It is clear that $\{H \in \mathcal{H}(M): cl_{M}(X) \subseteq H\} \subseteq \{H \in \mathcal{H}(M): X \subseteq H\}$ because $X \subseteq cl_{M}(X)$.
For all $H \in \{H \in \mathcal{H}(M): X \subseteq H\}$, $cl_{M}(X) \subseteq cl_{M}(H) = H$.
Thus $\{H \in \mathcal{H}(M): X \subseteq H\} \subseteq \{H\in \mathcal{H}(M): cl_{M}(X) \subseteq H\}$.
Therefore, $\{H \in \mathcal{H}(M): X \subseteq H\} = \{H \in \mathcal{H}(M): cl_{M}(X) \subseteq H\}$, i.e., $cl_{M}(X) = \bigcap \{H \in \mathcal{H}(M): cl_{M}(X) \subseteq H\} = \bigcap \{H \in \mathcal{H}(M): X \subseteq H\}$.
\end{proof}

Therefore, the properties of the closures of subsets can be equivalently described by hyperplanes.

\begin{proposition}
\label{P:ancharacterizesofclosedsets}
Let $M$ be a matroid of $U$ and $X, Y \subseteq U$.
Then $cl_{M}(X) \subseteq cl_{M}(Y) \Leftrightarrow \forall H \in \mathcal{H}(M)(Y \subseteq H \rightarrow X \subseteq H)$.
\end{proposition}

\begin{proof}
(``$\Rightarrow$"): If $cl_{M}(Y) = U$, then for all $H \in \mathcal{H}(M)$, $Y \nsubseteq H$.
Thus we have the result.
If $cl_{M}(Y) \neq U$, then $\{H \in \mathcal{H}(M): Y \subseteq H\} \neq \emptyset$.
Since $Y \subseteq H$, $X \subseteq cl_{M}(X) \subseteq cl_{M}(Y) \subseteq cl_{M}(H) = H$.
(``$\Leftarrow$"): If for all $H \in \mathcal{H}(M)$, $Y \nsubseteq H$, then $cl_{M}(Y) = U$.
Otherwise, $r_{M}(Y) \neq r_{M}(U)$ holds.
Utilizing Proposition \ref{P:theotherclosureexpressionbyhyperplane}, there exists $H \in \mathcal{H}(M)$ such that $Y \subseteq H$, which implies a contradiction.
Thus $cl_{M}(X) \subseteq U = cl_{M}(Y)$.
For all $H \in \mathcal{H}(M)$, if $Y \subseteq H$, then by assumption we have $\{H \in \mathcal{H}(M): Y \subseteq H\} \subseteq \{H \in \mathcal{H}(M): X \subseteq H\}$, which implies that $cl_{M}(X) \subseteq cl_{M}(Y)$.
\end{proof}

Based on the result, we find that the equivalence relation defined by Definition \ref{D:arelationinducedbymatroids} is a congruence relation.
Therefore, a dependence space of matroids is constructed.

\begin{theorem}
\label{T:thedependencespaceinducedbymatroid}
Let $M$ be a matroid of $U$.
Then $(U, \Theta_{M})$ is a dependence space.
\end{theorem}

\begin{proof}
Let us assume that $A_{1}, A_{2}, B_{1}, B_{2} \subseteq U$ satisfy $(A_{1}, A_{2}) \in \Theta_{M}$ and $(B_{1}, B_{2})$ $\in \Theta_{M}$.
Then $cl_{M}(A_{1})$ $ = cl_{M}(A_{2})$ and $cl_{M}(B_{1}) = cl_{M}(B_{2})$.
Now we need to prove $(A_{1} \bigcup B_{1},$ $A_{2} \bigcup B_{2}) \in \Theta_{M}$, i.e., $cl_{M}(A_{1} \bigcup B_{1}) = cl_{M}(A_{2} \bigcup B_{2})$.
Case 1: $cl_{M}(A_{1}) $ $ = cl_{M}(A_{2}) = U$ or $cl_{M}(B_{1}) = cl_{M}(B_{2}) = U$.
We may as well suppose $cl_{M}(A_{1}) = cl_{M}(A_{2}) = U$, then $U = cl_{M}(A_{1}) \subseteq cl_{M}(A_{1} \bigcup B_{1}) \subseteq cl_{M}(U) = U$ and $U = cl_{M}(A_{2}) \subseteq cl_{M}(A_{2} \bigcup B_{2}) \subseteq cl_{M}(U) = U$, i.e., $cl_{M}(A_{1} \bigcup B_{1}) = U = cl_{M}(A_{2} \bigcup B_{2})$.
Case 2: $cl_{M}(A_{1}) = cl_{M}(A_{2}) \neq U$ and $cl_{M}(B_{1}) = cl_{M}(B_{2}) \neq U$.
According to Proposition \ref{P:ancharacterizesofclosedsets}, we have $\{H \in \mathcal{H}(M): A_{2} \subseteq H\} = \{H \in \mathcal{H}(M): A_{1} \subseteq H\}$ and $\{H \in \mathcal{H}(M): B_{1} \subseteq H\} = \{H \in \mathcal{H}(M): B_{2} \subseteq H\}$.
If $cl_{M}(A_{1} \bigcup B_{1}) \neq U$, then $cl_{M}(A_{1} \bigcup B_{1}) = \bigcap \{H \in \mathcal{H}(M): A_{1} \bigcup B_{1} \subseteq H\} = \bigcap \{\{H \in \mathcal{H}(M): A_{1} \subseteq H\} \bigcap \{H \in \mathcal{H}(M): B_{1} \subseteq H\}\} = \bigcap \{\{H \in \mathcal{H}(M): A_{2} \subseteq H\} \bigcap \{H \in \mathcal{H}(M): B_{2} \subseteq H\}\} = \bigcap \{H \in \mathcal{H}(M): A_{2} \bigcup B_{2} \subseteq H\} = cl_{M}(A_{2} \bigcup B_{2})$.
If $cl_{M}(A_{1} \bigcup B_{1}) = U$, then we claim $cl_{M}(A_{2} \bigcup B_{2}) = U$.
If $cl_{M}(A_{2} \bigcup B_{2}) \neq U$, then $cl_{M}(A_{1} \bigcup B_{1}) = cl_{M}(A_{2} \bigcup B_{2}) \neq U$, which implies a contradiction.
Therefore, $\Theta_{M}$ is a congruence relation, i.e., $(U, \Theta_{M})$ is a dependence space.
\end{proof}

\section{Matroidal approach to dependence space}
\label{S:Matroidalapproachtodependencespaces}
In matroid theory, there are many greedy algorithms designed using bases of matroids.
In practice, the local optimal solutions obtained using these algorithms in matroidal structures are often global ones.
As is known, the algorithms for the issues of attribute reduction are almost greedy one.
These promote us to establish relationships between the bases of matroids and the reducts of dependence spaces.
In Section \ref{S:Dependencespaceinducedbymatroids}, we have constructed a dependence space by matroids.
In this section, we study the reduction problems of the dependence space.
It is interesting that the set of redusts of the dependence space is the set of bases of the corresponding matroid.
First, we present the concept of consistent sets.

\begin{definition}(Consistent set~\cite{NovotnyPawlak91Algebraic})
Let $(U, \Theta)$ be a dependence space.
A subset $X$ of $U$ is consistent in $(U, \Theta)$ if $X$ is minimal with respect to the inclusion relation in its $\Theta-$class.
Otherwise it is inconsistent.
The collection of consistent sets of $(U, \Theta)$ is denoted by $IND_{\Theta}$.
\end{definition}

In fact, the consistent sets of the dependence space induced by matroids can be equivalently characterized by other words.

\begin{proposition}
\label{P:anequivalencecharacterizationforindependentsetindependencespace}
Let $M$ be a matroid of $U$ and $X \subseteq U$.
$X \in IND_{\Theta_{M}}$ if and only if $(X, X - \{x\}) \notin \Theta_{M}$ for all $x \in X$.
\end{proposition}

\begin{proof}
According to Theorem \ref{T:thedependencespaceinducedbymatroid}, we know $(U, \Theta_{M})$ is a dependence space.
If $X \in IND_{\Theta_{M}}$, then obviously $(X, X - \{x\}) \notin \Theta_{M}$ for all $x \in X$.
Conversely, if $X \notin IND_{\Theta_{M}}$, then there exists $Y \subset X$ such that $(Y, X) \in \Theta_{M}$.
Since $Y \subset X$, there exists $x \in X - Y$ such that $Y \subseteq X - \{x\} \subset X$.
According to $(Y, X) \in \Theta_{M}$, $(X - \{x\}, X - \{x\}) \in \Theta_{M}$ and $(U, \Theta_{M})$ is a dependence space, then $(Y \bigcup (X - \{x\}), X \bigcup (X - \{x\})) = (X - \{x\}, X) \in \Theta_{M}$ which contradicts the assumption.
\end{proof}

From the equivalent characterization, we find that the consistent sets have a close relationship with the independent sets of matroids.
First, we review a well-known result of matroid theory.

\begin{lemma}\cite{Oxley93Matroid}
\label{L:onecharacterizationforindependentsetfromtheviewpointofclosureoperator}
Let $M$ be a matroid of $U$ and $X \subseteq U$.
$X \in \mathcal{I}(M)$ if and only if $x \notin cl_{M}(X - \{x\})$ for all $x \in X$.
\end{lemma}

For a matroid $M$, the following theorem indicates that any independent set of the matroid is a consistent set of the corresponding dependence space, and vice versa.

\begin{theorem}
\label{T:therelationshipbetweentheindependentsetsinmatroidandindependencespace}
Let $M$ be a matroid of $U$.
Then $IND_{\Theta_{M}} = \mathcal{I}(M)$.
\end{theorem}

\begin{proof}
Suppose $X \subseteq U$.
According to Proposition \ref{P:anequivalencecharacterizationforindependentsetindependencespace} and Lemma \ref{L:onecharacterizationforindependentsetfromtheviewpointofclosureoperator}, we need to prove that for all $x \in X$,
$(X, X - \{x\}) \notin \Theta_{M}$ if and only if $x \notin cl_{M}(X - \{x\})$.
First, we prove the sufficiency.
For all $x \in X$, $x \in cl_{M}(X)$ because $X \subseteq cl_{M}(X)$.
Combining with $x \notin cl_{M}(X - \{x\})$, we have $cl_{M}(X) \neq cl_{M}(X - \{x\})$, i.e., $(X, X - \{x\}) \notin \Theta_{M}$.
Conversely, if there exists $x \in X$ such that $x \in cl_{M}(X - \{x\})$, then $r_{M}(X) = r_{M}(X - \{x\})$ which implies $X \subseteq cl_{M}(X - \{x\})$, i.e., $cl_{M}(X) \subseteq cl_{M}(X - \{x\})$.
Since $X - \{x\} \subseteq X$, $cl_{M}(X - \{x\}) \subseteq cl_{M}(X)$.
Hence $cl_{M}(X) = cl_{M}(X - \{x\})$, i.e., $(X, X - \{x\}) \in \Theta_{M}$ which implies a contradiction.
\end{proof}

Based on the concept of consistent sets, reducts of a dependence space are defined.

\begin{definition}(Reduct of dependence space~\cite{NovotnyPawlak91Algebraic})
\label{D:thedefinitionofreductsofdependencespace}
Let $(U, \Theta)$ be a dependence space.
For all $X \subseteq U$, a subset $Y$ of $X$ is called a reduct of $X$, if $(X, Y) \in \Theta$ and $Y \in IND_{\Theta}$.
The set of all reducts of $X$ is denoted by $RED_{\Theta}(X)$.
\end{definition}

Theorem \ref{T:therelationshipbetweentheindependentsetsinmatroidandindependencespace} reveals the relationship between the consistent sets of dependence spaces and the independent sets of matroids.
Naturally, it motives us to connect the reducts of dependence spaces with other concepts of matroids.
At first, an important result in matroids is shown as follows.

\begin{lemma}\cite{Oxley93Matroid}
\label{L:therelationshipbetweentheclosedsetofanybaseofasetandtheset}
Let $M$ be a matroid of $U$ and $X \subseteq U$.
If $B_{X} \in \mathcal{B}(M|X)$, then $cl_{M}(B_{X}) = cl_{M}(X)$.
\end{lemma}

In fact, the reducts of dependence spaces have a close relationship with the bases of matroids.
For a given set, we find that the family of reducts of the set in the dependence space induced by a matroid is a collection of bases of the matroid imposed restriction on the set.

\begin{theorem}
\label{T:therelationbetweenreductsandbses}
Let $M$ be a matroid of $U$ and $X \subseteq U$.
$RED_{\Theta_{M}}(X) = \mathcal{B}(M|X)$.
\end{theorem}

\begin{proof}
According to the definition of bases of restriction matroids and Theorem \ref{T:therelationshipbetweentheindependentsetsinmatroidandindependencespace}, we know $\mathcal{B}(M|X) = Max(\{B \subseteq X: B \in IND_{\Theta_{M}}\})$.
For any $Y \in RED_{\Theta_{M}}(X)$, then $Y \in \{B \subseteq X: B \in IND_{\Theta_{M}}\}$.
Next, we need to prove $Y \in Max(\{B \subseteq X: B \in IND_{\Theta_{M}}\})$.
Otherwise, there exists $Y_{1} \in \{B \subseteq X: B \in IND_{\Theta_{M}}\}$ such that $Y \subset Y_{1}$.
Since $Y \in RED_{\Theta_{M}}(X)$, $(X, Y) \in \Theta_{M}$.
According to $(Y_{1}, Y_{1}) \in \Theta_{M}$ and $\Theta_{M}$ is a congruence relation, we have $(X \bigcup Y_{1}, Y \bigcup Y_{1}) \in \Theta_{M}$, i.e., $(X, Y_{1}) \in \Theta_{M}$.
Since $\Theta_{M}$ is an equivalence relation, $(Y, Y_{1}) \in \Theta_{M}$ which contradicts $Y_{1} \in IND_{\Theta_{M}}$.
Therefore, $RED_{\Theta_{M}}(X) \subseteq \mathcal{B}(M|X)$.
Conversely, since $\mathcal{B}(M|X) \subseteq \mathcal{I}(M)$, $\mathcal{B}(M|X) \subseteq IND_{\Theta_{M}}$.
According to Lemma \ref{L:therelationshipbetweentheclosedsetofanybaseofasetandtheset}, we know for all $Y \in \mathcal{B}(M|X)$, $cl_{M}(Y) = cl_{M}(X)$, i.e., $(X, Y) \in \Theta_{M}$.
Therefore $\mathcal{B}(M|X) \subseteq RED_{\Theta_{M}}(X)$.
\end{proof}

Finally, an example is provided to conclude this section.

\begin{example}
Suppose matroid $M$ is given in Example \ref{E:example1}.
Let $X = \{1, 3\}$.
Since $cl_{M}(\{1\}) = cl_{M}(\{3\}) = cl_{M}(\{1, 3\}) = \{1, 3\}$, $(X, \{1\}) \in \Theta_{M}$, $(X, \{3\}) \in \Theta_{M}$ and $(X, X) \in \Theta_{M}$.
Utilizing the definition of consistent sets and $cl_{M}(\emptyset) = \emptyset$, we know $\{1\} \in IND_{\Theta_{M}}$ and $\{3\} \in IND_{\Theta_{M}}$.
Therefore, $RED_{\Theta_{M}}(X) = \{Y \subseteq X: (X, Y) \in \Theta_{M}, Y \in IND_{\Theta_{M}}\} = \{\{1\}, \{3\}\} = Max(\{B \subseteq X: B \in \mathcal{I}(M)\}) = \mathcal{B}(M|X)$.
\end{example}
\section{Dependence space approach to matroid }
\label{S:Dependencespaceapproachtomatroids}
Dependence spaces introduced as a general abstract setting for studying informational dependency can be used to study the matroids.
In this section, we present two expressions for the bases of matroids in terms of closed sets through dependence spaces.
In fact, the reducts of dependence spaces can be equivalently characterized by the following form.

\begin{proposition}
\label{P:theotherexpressionofreductsindependencespace}
Let $(U, \Theta)$ be a dependence space and $X \subseteq U$.
Then $RED_{\Theta}(X) = Min(\{Y \subseteq X: (X, Y) \in \Theta\})$.
\end{proposition}

\begin{proof}
For all $B \in RED_{\Theta}(X)$, then $B \in IND_{\Theta}$ and $B \in \{Y \subseteq X: (X, Y) \in \Theta\}$.
If $B \notin Min(\{Y \subseteq X: (X, Y) \in \Theta\})$, then there exists $B_{1} \in \{Y \subseteq X: (X, Y) \in \Theta\}$ such that $B_{1} \subset B$.
Thus $(X, B_{1}) \in \Theta$.
Combining with $(X, B) \in \Theta$, then $(B_{1}, B) \in \Theta$ because $\Theta$ is transitive.
It contradicts that $B \in IND_{\Theta}$.
Therefore $RED_{\Theta}(X) \subseteq Min(\{Y \subseteq X: (X, Y) \in \Theta\})$.
Conversely, for all $B \in Min(\{Y \subseteq X: (X, Y) \in \Theta\})$, we need to prove that $B \in IND_{\Theta}$.
Otherwise, there exists $B_{1} \subseteq U$ such that $B_{1} \subset B$ and $(B, B_{1}) \in \Theta$.
Then $B_{1} \subseteq X$ and $(X, B_{1}) \in \Theta$ because $B \subseteq X$ and $\Theta$ is transitive.
That implies that $B_{1} \in \{Y \subseteq X: (X, Y) \in \Theta\}$ which contradicts the minimality of $B$.
Hence $Min(\{Y \subseteq X: (X, Y) \in \Theta\}) \subseteq RED_{\Theta}(X)$.
\end{proof}

Based on the above results, the following proposition presents an expression of bases from the viewpoint of closed set.
For a given subset $X$ of $U$, it is interesting that any base of the restriction of matroid $M$ to $X$ is minimal in the sense that it does not contain any other element of the family whose any element is contained in $X$ and has the same closure with $X$.

\begin{theorem}
\label{T:anexpressionofbasesbycloasedsets}
Let $M$ be a matroid of $U$ and $X \subseteq U$.
$\mathcal{B}(M|X) = Min(\{Y \subseteq X: cl_{M}(Y) = cl_{M}(X)\})$.
\end{theorem}

\begin{proof}
According to Theorem \ref{T:thedependencespaceinducedbymatroid}, \ref{T:therelationbetweenreductsandbses} and Proposition \ref{P:theotherexpressionofreductsindependencespace}, $\mathcal{B}(M|X) = RED_{\Theta_{M}}(X) = Min(\{Y \subseteq X: cl_{M}(Y) = cl_{M}(X)\})$.
\end{proof}

As is known, each family $\mathcal{H}$ of $2^{U}$ can define a congruence relation by defining
\begin{center}
    $\Gamma(\mathcal{H}) = \{(B_{1}, B_{2}) \in 2^{U} \times 2^{U}: \forall H \in \mathcal{H} (B_{1} \subseteq H \leftrightarrow B_{2} \subseteq H\})$.
\end{center}
Based on the relation, in~\cite{Novotny98Applications}, Novotn$\grave{y}$ introduce the concept of dense family.
It shows that the family $\mathcal{H}$ is dense in a dependence space $(U, \Theta)$ if and only if the congruence relation $\Gamma(\mathcal{H})$ defined by $\mathcal{H}$ equals to $\Theta$.
The second part of this section, we aim to present the other expression for bases of matroids through dependence spaces.
For the purpose, we take the hyperplane family $\mathcal{H}(M)$ to induce a congruence relation $\Gamma(\mathcal{H}(M))$.
The following example help to illustrate the congruence relation.

\begin{example}
Suppose matroid $M$ is given in Example \ref{E:example1}.
It is clear that $\mathcal{H}(M) = \{\{2\}, \{1, 3\}\}$.
Let $B_{1} = \{1\}$ and $B_{2} = \{3\}$.
Then $(B_{1}, B_{2}) \in \Gamma(\mathcal{H}(M))$.
Because when we take $H = \{2\}$, $B_{1} \nsubseteq H$ and $B_{2} \nsubseteq H$ which implies $B_{1} \subseteq H \leftrightarrow B_{2} \subseteq H$ holds.
When we take $H = \{1, 3\}$, $B_{1} \subseteq H$ and $B_{2} \subseteq H$ which implies $B_{1} \subseteq H \leftrightarrow B_{2} \subseteq H$ holds.
Let $B_{1} = \{2\}$ and $B_{2} = \{1, 2\}$.
Then $(B_{1}, B_{2}) \notin \Gamma(\mathcal{H}(M))$.
Because when we take $H = \{2\}$, $B_{1} \subseteq H$ and $B_{2} \nsubseteq H$ which implies $B_{1} \subseteq H \rightarrow B_{2} \subseteq H$ dose not hold.
By this way, we can obtain $\Gamma(\mathcal{H}(M)) = \Theta$.
\end{example}

Inspired by the example, we find that the congruence relation $\Gamma(\mathcal{H}(M))$ is dense in the dependence space $(U, \Theta_{M})$.

\begin{proposition}
\label{P:thedensefamilyofthedependencespaceinducedbymatroid}
Let $M$ be a matroid of $U$.
Then $\mathcal{H}(M)$ is dense in $(U, \Theta_{M})$.
\end{proposition}

\begin{proof}
By the definition of dense family, we need to prove $\Theta_{M} = \Gamma(\mathcal{H}(M))$.
Utilizing Proposition \ref{P:ancharacterizesofclosedsets}, we have proved the result.
\end{proof}

For a dependence space, the redusts of a given subset can be characterized by its dense family through the following way.

\begin{lemma}\cite{ZhangLiangWu03Information}
\label{L:acharacterizationforreductbydensefamily}
Let $\Gamma$ be dense in a dependence space $(U, \Theta)$.
For all $X \subseteq U$, $RED_{\Theta}(X) = Min(\{B \subseteq U: B \bigcap T \neq \emptyset~(\forall T \in Com(\Gamma)\})$, where $Com(\Gamma) = \{T \neq \emptyset, X - T \in \Gamma\}$.
\end{lemma}

Therefore, we present the other expression for bases of matroids in terms of hyperplanes.
In fact, for any element of $\mathcal{B}(M|X)$, it is a minimal set with respect to the property of containing at leat one element from each nonempty difference $X - H$, where $H$ is a hyperplane of $M$.

\begin{theorem}
Let $M$ be a matroid of $U$ and $X \subseteq U$.
Then $\mathcal{B}(M|X) = Min(\{B \subseteq U: B \bigcap T \neq \emptyset~(\forall T \in Com(\mathcal{H}(M))\})$, where $Com(\mathcal{H}(M))\}) = \{T \neq \emptyset, X - T \in \mathcal{H}(M))\}\}$.
\end{theorem}

\begin{proof}
The result is obtained by combining Proposition \ref{P:thedensefamilyofthedependencespaceinducedbymatroid} with Lemma \ref{L:acharacterizationforreductbydensefamily}.
\end{proof}

Finally, we show an example to conclude this section.

\begin{example}
Suppose the matroid $M$ is given in Example \ref{E:example1} and $X = \{1, 3\}$.
By Example \ref{E:example1}, $cl_{M}(\{1, 3\}) = cl_{M}(\{1\}) = cl_{M}(\{3\}) = \{1, 3\}$.
On one hand, $Min(\{Y \subseteq X: cl_{M}(Y) = cl_{M}(X)\}) = \{\{1\}, \{3\}\}$.
On the other hand, since $\mathcal{H}(M) = \{\{2\}, \{1, 3\}\}$, $Com(\mathcal{H}(M)) = \{T \neq \emptyset, X - T \in \mathcal{H}(M)\} = \{\{1, 3\}\}$.
Thus $Min(\{T \subseteq U: B \bigcap T \neq \emptyset~(\forall T \in Com(\mathcal{H}(M))\}) = Min(\{\{1\}, \{3\}, \{1, 3\}\}) = \{\{1\}, \{3\}\}$.
Therefore $\mathcal{B}(M|X) = Min(\{Y \subseteq X: cl_{M}(Y) = cl_{M}(X)\}) = Min(\{B \subseteq U: B \bigcap T \neq \emptyset~(\forall T \in Com(\mathcal{H}(M))\})$.
\end{example}

\section{Conclusions}
\label{S:conclusions}
In this paper we proposed a dependence space in the context of matroids.
Then the dependence space were studied by matroids, and vice versa.
Based on the results of this paper, we intend to design efficient algorithms to solve attribute reduction issues in information systems and will investigate these issues using a variety of other theories such as geometric lattices.

\section{Acknowledgments}
This work is supported in part by the National Natural Science Foundation of China under Grant Nos. 61170128, 61379049, and 61379089, the Natural Science Foundation of Fujian Province, China, under Grant No. 2012J01294, the Science and Technology Key Project of Fujian Province, China, under Grant No. 2012H0043, and the Zhangzhou Research Fund under Grant No. Z2011001.



\end{document}